\documentclass[11pt]{article}

\usepackage{enumerate,xspace}
\usepackage{amsmath,xspace,amssymb,stmaryrd}
\usepackage{amsthm}
\usepackage[dvips]{graphics}
\usepackage{algorithm}
\usepackage{algpseudocode}

\input xy
\xyoption{all}
\xyoption{2cell}
\UseAllTwocells
\CompileMatrices

\title{Merging and Comparing Ontologies Algebraically}
\author{Xiuzhan Guo, Arthur Berrill, Ajinkya Kulkarni, Kostya Belezko, Min Luo}

\date{}

\usepackage{enumerate,xspace}
\usepackage{amsmath,xspace,amssymb}
\usepackage{amsthm}
\usepackage[dvips]{graphics}
\usepackage{stmaryrd}
\usepackage{verbatim}
\usepackage{fullpage}
\usepackage{amsfonts}
\usepackage{proof}
\usepackage[toc, page]{appendix}

\input xy
\xyoption{all}
\xyoption{2cell}

\UseAllTwocells
\CompileMatrices

\newcommand{\rw}{\rightarrow}

\newcommand{\cupdot}{\mathbin{\mathaccent\cdot\cup}}

\newtheorem{theorem}{Theorem}[section]    
\newtheorem*{theorem*}{Theorem}

\newtheorem{lemma}[theorem]{Lemma}   
\newtheorem{preremark}[theorem]{Remark}   
\newtheorem{prexample}[theorem]{Example}   
\newtheorem{proposition}[theorem]{Proposition}

\newtheorem{definition}[theorem]{Definition}

\newenvironment{remark}{\begin{preremark}\rm}{\end{preremark}}
\newenvironment{example}{\begin{prexample}\rm}{\end{prexample}}

\begin{document}

\maketitle

\begin{abstract}
Ontology operations, e.g., aligning and merging, were studied
and implemented extensively in different settings, such as, categorical operations, relation algebras, 
typed graph grammars, with different concerns.
However, aligning and merging operations in the settings share some generic properties, e.g., 
idempotence, commutativity, associativity, and representativity, labeled by
$\eqref{eqn:i}$,  $\eqref{eqn:c}$, $\eqref{eqn:a}$, $\eqref{eqn:r}$, respectively,
which are defined on an ontology merging system $(\mathfrak{O},\sim,\merge)$,
where $\mathfrak{O}$ is a non-empty set of the ontologies concerned, 
$\sim$ is a binary relation on $\mathfrak{O}$ modeling ontology aligning 
and $\merge$ is a partial binary operation on $\mathfrak{O}$ modeling ontology merging.
Given an ontology repository, a finite set $\mathbb{O}\subseteq \mathfrak{O}$, 
its merging closure $\widehat{\mathbb{O}}$ is the smallest set of ontologies, which contains the repository
and is closed with respect to merging.
If $\eqref{eqn:i}$, $\eqref{eqn:c}$, $\eqref{eqn:a}$, and $\eqref{eqn:r}$ are satisfied,
then both $\mathfrak{O}$ and $\widehat{\mathbb{O}}$ are partially ordered naturally by merging
so that the ontologies can be compared, sorted, and selected,
$\widehat{\mathbb{O}}$ is finite and can be computed efficiently, including sorting, selecting, 
and querying some specific elements, e.g., maximal ontologies and minimal ontologies.
We also show that the ontology merging system, given by ontology $V$-alignment pairs and pushouts,
satisfies the properties: $\eqref{eqn:i}$, $\eqref{eqn:c}$, $\eqref{eqn:a}$, and $\eqref{eqn:r}$ so that
the merging system is partially ordered and the merging closure of a given repository with respect to pushouts 
can be computed efficiently.
\end{abstract}


\pagenumbering{arabic}


\section{Introduction}
Ontology has been playing important roles in many areas, 
such as, database integration, peer-to-peer systems, e-commerce, semantic web services, social networks, etc.
Individual ontologies, however, are not enough to provide the full functionalities.
To reuse existing ontologies, one usually wants
to integrate different ontologies together. This can be done by combining
the ontologies in which they are merged, operated and transformed into one new ontology. 
In this case, the ontologies have to be aligned with mutual agreement 
from different ontologies with heterogeneous specifications \cite{zkeh}.

Ontology operations, e.g., aligning and merging, are studied
and implemented extensively in different settings, such as, 
categorical operations \cite{aa, ch, cmk14, cmk17, hw, keg, map, zkeh}, relation algebras \cite{e}, 
typed graph grammars \cite{mtfh}.
In \cite{zkeh}, the concept of ontology alignment was formalized by an ontology $V$-alignment pair 
$(r_1:B\rw O_1, r_2:B\rw O_2)$ with ontology homomorphisms $r_1:B\rw O_1$ and $r_2:B\rw O_2$
while merging two ontologies $O_1$ and $O_2$, for which there is an ontology $V$-alignment pair 
$(r_1:B\rw O_1, r_2:B\rw O_2)$, was obtained by the categorical colimit, 
the pushout of $r_1$ and $r_2$, in the category of the ontologies concerned there.

Recall that a {\em groupoid} $(G,\circ)$ is a non-empty set $G$ equipped with a binary operation $\circ$ that is a function from 
$G\times G$ to $G$ and a {\em semigroup} $(G,\circ)$ is a groupoid $(G,\circ)$ with an {\em associative} binary operation $\circ$.
In some cases, some binary operations on a non-empty set $G$ is {\em partial}, namely, 
the operations are defined only on a subset of $G\times G$.
For example, given two matrices $M$ and $N$,
$M\cdot N$ is defined only when the number of columns in $M$ equals the number of rows in $N$.
Ontology merging can be viewed as a partial binary operation that is defined only on ontology $V$-alignment pairs
and so the ontologies concerned, along with their merging, form a partial algebraic structure, e.g., a partial groupoid or a partial semigroup, depending on the associativity of the merging operation.
Hence we can study the ontology aligning and merging operations together within the partial groupoid or semigroup 
using the properties the operations share.

A database can grow quickly from a finite set of data tables by generating new tables 
using the table operations, such as, join, select, combine, etc.
Given an ontology {\em repository} or {\em instance} 
$\mathbb{O}$, a finite set of ontologies, it is natural to ask {\em how many new ontologies can be generated from the repository
by merging and how to compare and rank these ontologies.}
The main objective of this paper is to answer the natural questions by studying the properties ontology aligning and merging operations share and {\em ontology merging closures} defined in Section \ref{section:closures} below.
Even though ontology aligning and merging operations have been formalized in different settings, we shall not study the ontology  aligning and merging operations with
any specific setting and the internal details of the operations 
but we shall use the generic algebraic properties, introduced in Section  \ref{section:properties} below, 
ontology aligning and merging share, to provide
a generic setting for studying ontology aligning and merging operations.

The rest of the paper is organized as  follows:
First, in Section \ref{section:properties},
we introduce ontology merging systems and investigate the properties defined on an ontology merging system,
which ontology aligning and merging operations share, such as, 
idempotence, commutativity, associativity, and representativity, denoted by 
$\eqref{eqn:i}$, $\eqref{eqn:c}$, $\eqref{eqn:a}$, $\eqref{eqn:r}$, respectively,
described in Example \ref{exam:mergingsys}.

In a given database, we may concern some specific tables, e.g.,
the tables containing maximal information, or smallest atomic tables.
Given a set of ontologies, to query some specific ontologies from the set of the ontologies concerned,
we need to compare and rank the ontologies in the set by a {\em partial order} on the set.
If the merging operation is associative, then the set of the ontologies concerned, along with merging, 
forms a partial semigroup. Following partial order approaches on semigroups, in Section \ref{section:naturalorder} we define 
the {\em natural partial order} given by the merging operation 
(Proposition \ref{proposition:naturalposet}).
We show that a given partial order on the ontologies must be the natural partial order if some reasonable properties are satisfied (Theorem \ref{thm:partialordersequal}).

In Section \ref{section:closures},
we view the repository $\mathbb{O}$ of ontologies as the generators in a partial algebraic environment and 
treat the merge closure of $\mathbb{O}$, the smallest set of ontologies, which contains $\mathbb{O}$ 
and is closed with respect to merging, 
as the subalgebraic structure generated by the repository $\mathbb{O}$.
We prove that the merging closure of a given repository is partially ordered and finite
under some reasonable conditions (Theorem \ref{thm:finiteness}) 
so that we can compare, sort, and select these ontologies
generated from the repository, using sorting and selection algorithms in partially ordered sets.

Section \ref{section:pushouts} shows that the ontology merging system, given by ontology $V$-alignment and pushouts,
satisfies the algebraic properties introduced in 
Sections \ref{section:properties} and \ref{section:naturalorder} (Theorem \ref{theorem:pushoutsall}) 
so that there are efficient algorithms to 
compute merging closures and query specific ontologies generated by pushouts.

Finally, we provide our concluding remarks in Section \ref{section:conclusions}.

\section{Ontology Merging Systems and Properties}\label{section:properties}
Ontology aligning and merging were studied and implemented in different settings extensively 
\cite{cmk14, cmk17, e, es, hkes, mfh, mtfh, szs, zkeh}.
However, the aligning and merging operations in the different settings share some algebraic properties.
In this section, we shall model ontology aligning and merging operations together by {\em ontology merging systems} and 
introduce the algebraic properties, defined on an ontology merging system,  ontology aligning and merging operations share.
We shall not specify any ontology setting with ontology internal details.

Let $\mathfrak{O}$ be the non-empty set of the ontologies concerned, e.g., 
the set of the objects of the category $\mathfrak{Ont}^+$ of the ontologies defined in \cite{zkeh},
the ontology structures considered in \cite{ch}, or the geospatial data ontologies in \cite{szphwls}.
Let $\sim$ be a binary relation on $\mathfrak{O}$ that models a generic ontology alignment relationship
and $\merge$ a partial binary operation on $\mathfrak{O}$ that models a merging operation defined on alignment pairs:
For all $O_1,O_2\in \mathfrak{O}$, $O_1\merge O_2$ exists if $O_1\sim O_2$
and $O_1\merge O_2$ is undefined, denoted by $O_1\merge O_2 =\;\uparrow$, otherwise.
$(\mathfrak{O}, \sim, \merge)$ forms an {\em ontology merging system}.
$(\mathfrak{O}, \sim, \merge)$  is {\em total} if $O_1\merge O_2\neq \;\uparrow$ for all $O_1,O_2\in \mathfrak{O}$.
Throughout, $(\mathfrak{O}, \sim,\merge)$ is an ontology system. 

The {\em null extension} $(\mathfrak{O}^{\uparrow}, \sim^{\uparrow}, \merge^{\uparrow})$ 
of $(\mathfrak{O}, \sim, \merge)$ is defined by 
\begin{itemize}
\item
$\mathfrak{O}^{\uparrow}=\mathfrak{O} \cupdot \{\uparrow \}$;
\item
For all $O_1,O_2,O\in \mathfrak{O}$, if $O_1\sim O_2$ then $O_1\sim^{\uparrow} O_2$, 
$O\nsim^{\uparrow} \;\uparrow\;\nsim^{\uparrow} O$, and $\uparrow \;\sim^{\uparrow} \;\uparrow$;
\item
For all $O_1,O_2\in \mathfrak{O}$ such that $O_1\sim O_2$, 
$O_1\merge^{\uparrow} O_2=O_1\merge O_2$. For all $O_1,O_2,O\in \mathfrak{O}$ such that $O_1\nsim O_2$,
$$O_1\merge^{\uparrow}O_2=O_2\merge^{\uparrow}O_1=O\;\merge^{\uparrow}\uparrow \;= \;\uparrow \merge^{\uparrow}O
=\;\uparrow\merge^{\uparrow}\uparrow \;= \;\uparrow.$$
\end{itemize}

We begin with the following algebraic properties on $(\mathfrak{O}, \sim, \merge)$:
\begin{description}
\item[{Idempotence:}] 
For all $O\in \mathfrak{O}$,
\begin{equation}\label{eqn:i}
\mbox{$O \sim O$ and $O\merge O  = O$.}\tag{I} 
\end{equation}
An ontology always aligns with itself, and merging it with itself yields the same ontology.
\item[{Commutativity:}] 
For all $O_1, O_2\in \mathfrak{O}, O_1 \sim O_2 $ if and only if $O_2 \sim O_1$, and if $O_1 \sim O_2$,  
\begin{equation}\label{eqn:c}
O_1\merge O_2 = O_2\merge O_1.\tag{C}
\end{equation}
Aligning is symmetric and merging does not depend on the order of the ontologies being merged.
\end{description}
In general, the associativity of a binary operation means that rearranging the parentheses in an expression will  produce 
the same result.
There are many ways to specify the associative property.
In this paper, we are interested in the following three associative properties:
$\eqref{eqn:a}$, $\eqref{eqn:cass}$, and $\eqref{eqn:sass}$.
\begin{description}
\item[{Associativity in \cite{BGSWW}:}] 
For all $O_1 , O_2 , O_3\in \mathfrak{O}$ such that $O_1\merge (O_2\merge O_3)$ and $(O_1\merge O_2)\merge O_3$ exist,
\begin{equation}\label{eqn:a}
O_1\merge (O_2\merge O_3)=(O_1\merge O_2)\merge O_3\neq \;\uparrow.\tag{A}
\end{equation}
Merging three ontologies $O_1$, $O_2$, $O_3$ yields the same result regardless of the orders in which $O_1$, $O_2$, $O_3$
are merged when both $O_1\merge (O_2\merge O_3)$ and $(O_1\merge O_2)\merge O_3$ exist.
\item[{Catenary associativity:}]
For all $O_1, O_2, O_3 \in \mathfrak{O}$ such that $O_1\merge O_2$ and $O_2\merge O_3$ exist,
\begin{equation}\label{eqn:cass}
(O_1\merge O_2)\merge O_3=O_1\merge (O_2\merge O_3)\neq \;\uparrow.\tag{CA}
\end{equation}
Merging three ontologies $O_1$, $O_2$, $O_3$ yields the same result regardless of the orders in which $O_1$, $O_2$, $O_3$
are merged when both $O_1\merge O_2$ and $O_2\merge O_3$ exist.
\item[{Strong associativity:}]
For all $O_1, O_2, O_3\in \mathfrak{O}$,
\begin{equation}\label{eqn:sass}
(O_1\merge O_2)\merge O_3=O_1\merge (O_2\merge O_3)\tag{SA}
\end{equation}
in the sense that both sides of the equation yield the same value or both are simultaneously undefined.
It is equivalent to
either of $(O_1\merge O_2)\merge O_3$ and $O_1\merge (O_2\merge O_3)$ is defined so is the other and they are equal.
\item[{Left representativity:}] 
Let $O_1,O_2, O_3 \in \mathfrak{O}$ such that $O_3 = O_1\merge O_2$.
\begin{equation}\label{eqn:rl}
\mbox{For all $O\in \mathfrak{O}$ such that  $O\sim O_1$, $O\sim O_3$.}\tag{$\mbox{R}_l$}
\end{equation}
\item[{Right representativity:}] 
Let $O_1,O_2, O_3 \in \mathfrak{O}$ such that $O_3 = O_1\merge O_2$.
\begin{equation}\label{eqn:rr}
\mbox{For all $O\in \mathfrak{O}$ such that  $O_1 \sim O$, $O_3 \sim O$.}\tag{$\mbox{R}_r$}
\end{equation}
\item[{Representativity:}] 
\begin{equation}\label{eqn:r}
\mbox{R$_l$ and R$_r$.}\tag{R}
\end{equation}
The ontology $O_3$ obtained from merging ontologies $O_1$ and $O_2$,
represents the original $O_1$ and $O_2$:
the ontology $O$ that has matched $O_1$ will also match $O_3$ from both left and right sides.
\end{description}

Some relationships of the above ontology properties are summarized in the following proposition.
The proofs for the proposition and subsequent results will be provided in Appendices $\ref{append:lemmas}$ and $\ref{append:theorems}$.
\begin{proposition}\label{proposition:assrel}
Let $(\mathfrak{O}, \sim,\merge)$ be an ontology merging system.
\begin{enumerate}[$(i)$]
\item\label{remark: assimplya}
If $(\mathfrak{O},\sim,\merge)$ satisfies $\eqref{eqn:sass}$, then $(\mathfrak{O},\sim,\merge)$ satisfies $\eqref{eqn:a}$;
\item\label{remark: cassimplya}
If $(\mathfrak{O},\sim,\merge)$ satisfies $\eqref{eqn:cass}$, then $(\mathfrak{O},\sim,\merge)$ satisfies $\eqref{eqn:a}$;
\item\label{remark: assequiv}
If $(\mathfrak{O},\sim,\merge)$ satisfies $\eqref{eqn:c}$, then
$\eqref{eqn:cass}$ is equivalent to $\eqref{eqn:a}$ and $\eqref{eqn:r}$;
\item\label{remark: assextension}
$(\mathfrak{O},\sim,\merge)$ satisfies $\eqref{eqn:sass}$ if and only if the null extension $(\mathfrak{O}^{\uparrow}, \sim^{\uparrow}, \merge^{\uparrow})$ is a total ontology merging system such that
$(O_1\merge^{\uparrow}O_2) \merge^{\uparrow}O_3=O_1 \merge^{\uparrow}(O_2 \merge^{\uparrow} O_3)$ for all $O_1,O_2,O_3\in \mathfrak{O}^{\uparrow}$, namely, $(\mathfrak{O}^{\uparrow},\merge^{\uparrow})$ is a semigroup.
\end{enumerate}
\end{proposition}

\begin{example}\label{exam:mergingsys}
\begin{enumerate}[$1.$]
\item
Define $(\mathfrak{O},\asymp,\cupdot)$ by
\begin{itemize}
\item
For all $O_1,O_2\in \mathfrak{O}$, $O_1\asymp O_2$,
\item
$O_1\cupdot O_2$ is the disjoint union of $O_1$ and $O_2$.
\end{itemize}
It is straightforward to verify that $(\mathfrak{O},\asymp,\cupdot)$ is an ontology system satisfying 
$\eqref{eqn:c}$, $\eqref{eqn:a}$, $\eqref{eqn:cass}$, $\eqref{eqn:sass}$, and $\eqref{eqn:r}$ but not $\eqref{eqn:i}$.
\item
Let $\mathfrak{G}$ be the set of graphs that represent the ontologies concerned as in \cite{mtfh}. Define 
$(\mathfrak{G},\sim,\between)$ by
\begin{itemize}
\item
For all $G_1,G_2\in \mathfrak{G}$, $G_1\sim G_2$ if and only if there exists an overlapping subgraph $S$ (up to graph isomorphism) 
of $G_1$ and $G_2$,
\item
$G_1\between G_2$ is the union of $G_1$ and $G_2$ by identifying their overlapping subgraph $S$ when $G_1\sim G_2$.
\end{itemize}
Then $(\mathfrak{G},\sim,\between)$ is an ontology merging system satisfying 
$\eqref{eqn:i}$, $\eqref{eqn:c}$, $\eqref{eqn:a}$, $\eqref{eqn:cass}$, $\eqref{eqn:sass}$, and $\eqref{eqn:r}$.
\item
In \cite{szphwls},
Sun et al. defined
$${\bf GeoDataOnt}=\{ (E, R_{(E_i,E_j)})\;|\;E_i,E_j\in E,0\leq i,j\leq |E|\},$$
where $E$ is the set of geographic entities concerned
and $R$ the set of relations between the entities from $E$.
${\bf GeoDataOnt}$ can be represented as a table of geographic entities and their relations, labeled by $E$ and $R$, respectively.
Let $\mathfrak{Geo}$ be set of such tables. For $\mathbb{G}_1,\mathbb{G}_2\in \mathfrak{Geo}$, define
$$\mathbb{G}_1\smile \mathbb{G}_2 \mbox{ for all }\mathbb{G}_1,\mathbb{G}_2 \in \mathfrak{Geo}$$ 
and
$$\mathbb{G}_1\Join \mathbb{G}_2= \mathbb{G}_1\mbox{ full join } \mathbb{G}_2.$$
Then $(\mathfrak{Geo}, \smile,\Join)$ is an ontology system satisfying all above properties: $\eqref{eqn:i}$, $\eqref{eqn:c}$, $\eqref{eqn:a}$, $\eqref{eqn:cass}$, $\eqref{eqn:sass}$, and $\eqref{eqn:r}$. 
\item
Let $(\mathfrak{O},\approx,\sqcup)$ be the ontology system, which we shall define formally in Definition \ref{def:valign} below, given by 
\begin{itemize}
\item
For all $O_1,O_2\in \mathfrak{O}$, $O_1\approx_B O_2$ if and only if there is a pair of ontology homomorphisms
$(r_1:B\rw O_1,r_2:B\rw O_2)$;
\item
If $O_1\approx_B O_2$, then $O_1\sqcup_BO_2$ is given by the pushout of $r_1$ and $r_2$.
\end{itemize} 
Then  $(\mathfrak{O},\approx,\sqcup)$ satisfies 
$\eqref{eqn:i}$, $\eqref{eqn:c}$, $\eqref{eqn:a}$, $\eqref{eqn:cass}$, $\eqref{eqn:sass}$, and $\eqref{eqn:r}$, which will be proved 
in Theorem \ref{theorem:pushoutsall} below.
\item
In \cite{ch}, Cafezeiro and Haeusler defined an ontology homomorphism between ontology structures introduced in \cite{ms}, 
as a pair of functions $(f,g)$, where $f$ is a function between the concepts
and $g$ a function between relations, which preserve the ontology structures.
They aligned ontologies by $V$-alignment pairs and merged aligned ontologies by pushouts in the category of
ontology structures.
Similarly, by Theorem \ref{theorem:pushoutsall} below,
the ontology merging system, satisfies
$\eqref{eqn:i}$, $\eqref{eqn:c}$, $\eqref{eqn:a}$, $\eqref{eqn:cass}$, $\eqref{eqn:sass}$, and $\eqref{eqn:r}$. 
\end{enumerate}
\end{example}

\section{Natural Partial Order Given by Merging}\label{section:naturalorder}
A partial order on a semigroup $(S,\circ)$ is called {\em natural} if it is defined by means of the operation $\circ$.
Recall that the set $E_S$ of all idempotents of $S$ is partially ordered by
$$e\leq_E f \mbox{ if and only if }e=e\circ f=f\circ e.$$
In \cite{mit}, Mitsch defined a natural partial order $\leq_M$ that extends the ordering $\leq_E$ from $E_S$ to $S$ by:
$$\mbox{for all }a,b \in S, a\leq_M b\mbox{ if and only if there exist }s,t\in S^1\mbox{ such that }a = s\circ b=b\circ t, s\circ a=a,$$
where $S^1$ is the semigroup obtained by adjoining the identity 1 to $S$.
Mitsch proved that $\leq_M$ is equal to $\leq_E$ when restricted to $E_S$ \cite{mit}.

Ontology merging $\merge$ aims to obtain more information by combining the aligned ontologies together. 
The natural ontology partial order $O_1\leq_{\merge} O_2$ is defined 
if merging $O_1$ to $O_2$ does not yield the more information than $O_2$.
Clearly $\leq_{\merge}$ is the dual order of $\leq_E$ when restricted to the idempotents.
\begin{definition}
For all $O_1, O_2\in \mathfrak{O}$, $O_1\leq_{\merge} O_2$ if and only if $O_1\sim O_2$,  $O_2\sim O_1$, and $O_1\merge O_2 =O_2\merge O_1= O_2$.
\end{definition}

The following proposition shows that $(\mathfrak{O},\leq_{\merge})$ is a partially ordered set (poset), namely, $\leq_{\merge}$  is 
a reflexive, antisymmetric, and transitive binary relation on $\mathfrak{O}$, when $\eqref{eqn:i}$ and $\eqref{eqn:cass}$ are satisfied.
\begin{proposition}\label{proposition:naturalposet}
If $(\mathfrak{O}, \sim, \merge)$ satisfies $\eqref{eqn:i}$ and $\eqref{eqn:cass}$,
then $\leq_{\merge}$ is a partial order on $\mathfrak{O}$ and so $(\mathfrak{O},\leq_{\merge})$ is a poset.
\end{proposition}

Since $(\mathfrak{O},\leq_{\merge})$ is a poset if $\eqref{eqn:i}$ and $\eqref{eqn:cass}$ are satisfied,
we can compare ontologies in $\mathfrak{O}$ and apply the sorting and selection algorithms 
in posets and for transitive relations \cite{dkmrv} 
to query the specific ontologies in $\mathfrak{O}$, e.g., maximal ontologies, minimal ontologies.

Let $\preceq$ be a given partial order on $\mathfrak{O}$. 
We consider the following properties with respect to $\preceq$:
\begin{description}
\item[{Merge gives the least upper bound:}]
For all ontologies $O_1$, $O_2$ such that $O_1\sim O_2$,
\begin{equation}\label{eqn:lub}
O_1\merge O_2\mbox{ is the least upper bound of }O_1\mbox{ and }O_2\mbox{ with respect to }\preceq.\tag{LU}
\end{equation}
\item[{$\preceq$ is left compatible with $\sim$ and $\merge$}:]
For all ontologies $O_1$, $O_2$, $O$ such that $O_1\preceq O_2$ and $O\sim O_1$, 
\begin{equation}\label{eqn:lcomp}
\mbox{$O\sim O_2$ and $O\merge O_1\preceq O\merge O_2$.} \tag{CP$_l$}
\end{equation}
\item[{$\preceq$ is right compatible with $\sim$ and $\merge$}:]
For all ontologies $O_1$, $O_2$, $O$ such that $O_1\preceq O_2$ and $O_1\sim O$, 
\begin{equation}\label{eqn:rcomp}
\mbox{$O_2\sim O$ and $O_1\merge O\preceq O_2\merge O$.} \tag{CP$_r$}
\end{equation}
\item[{$\preceq$ is compatible with $\sim$ and $\merge$}:] If $\preceq$ is 
\begin{equation}\label{eqn:comp}
\mbox{both left and right compatible with $\sim$ and $\merge$.}\tag{CP}
\end{equation}
\end{description}

A partial order on $(\mathfrak{O},\sim,\merge)$, where $\sim$ is reflexive and commutative,  must be the natural partial order 
$\leq_{\merge}$ if 
$\eqref{eqn:lub}$ and $\eqref{eqn:comp}$ are satisfied shown in the following theorem.
\begin{theorem}\label{thm:partialordersequal}
Let $\sim$ be reflexive and  commutative and $\preceq$ a partial order on $(\mathfrak{O},\sim,\merge)$. Then 
$(\mathfrak{O},\sim,\merge)$ satisfies $\eqref{eqn:lub}$ and $\eqref{eqn:comp}$ with respect to $\preceq$ if and only if 
$(\mathfrak{O},\sim,\merge)$ satisfies $\eqref{eqn:i}$, $\eqref{eqn:c}$, $\eqref{eqn:a}$, 
and $\eqref{eqn:r}$ and $\preceq\;=\;\leq_{\merge}$.
\end{theorem}

\section{Ontology Merging Closures and Finiteness}\label{section:closures}
An ontology {\em repository or instance} in 
$(\mathfrak{O},\sim,\merge)$ is a finite set $\mathbb{O}\subseteq \mathfrak{O}$. 
It is natural to ask: {\em How many new ontologies can be generated from $\mathbb{O}$ by merging operation 
$\merge$ and how to rank and compare the ontologies?}
In this section, we shall provide our answers to the questions by introducing the merging closure of $\mathbb{O}$.
\begin{definition}
Given a repository $\mathbb{O}\subseteq \mathfrak{O}$, the {\em merging closure} of $\mathbb{O}$,
denoted by $\widehat{\mathbb{O}}$,
is the smallest set $\mathbb{P}\subseteq \mathfrak{O}$ such that
\begin{enumerate}[$1.$]
\item
$\mathbb{O}\subseteq \mathbb{P}$,
\item
$\mathbb{P}$ is closed with respect to merging: for all $O_1, O_2\in \mathbb{P}$ such that $O_1\sim O_2$, 
$O_1\merge O_2\in \mathbb{P}$.
\end{enumerate}
\end{definition}

The existence and uniqueness of the merging closure of a given ontology repository are given by the following theorem.
\begin{theorem}\label{thm:mergeclosure}
Given a repository $\mathbb{O}\subseteq \mathfrak{O}$, 
the merging closure $\widehat{\mathbb{O}}$ exists and it is unique.
\end{theorem}

Constructively, $(\mathfrak{O},\merge)$ being considered as a groupoid, $\widehat{\mathbb{O}}$ is
the partial subgroupoid generated by $\mathbb{O}$ with the partial binary operation $\merge$. Hence
$$\widehat{\mathbb{O}}=\bigcup_{i=1}^{+\infty}\mathbb{O}^i,$$
where $\mathbb{O}^i$ is defined inductively by:
\begin{itemize}
\item
If $n=1$, $\mathbb{O}^n=\mathbb{O}$.
\item
If $n>1$,
\[ O\in \mathbb{O}^n \mbox{ if and only if }  \left\{ \begin{array}{ll}
                           \mbox{there is a positive integer } k: 1\leq k< n\\
                           \mbox{there is } O_1\in \mathbb{O}^k \\
                           \mbox{there is }O_2\in \mathbb{O}^{n-k}\\
                           \end{array}\right\}O=O_1\merge O_2 \]
\end{itemize}

Obviously, the product $\mathbb{O}^n$ is the set of all possible values 
$$O_1\merge\cdots (O_i\merge O_{i+1})\cdots \merge O_n,$$ 
where all $O_1\merge \cdots O_i\cdots \merge O_n$ are binarily grouped with the parentheses ``('' and ``)'' being inserted and 
$O_i\in \mathbb{O}, i=1,\cdots,n$.

To compute $\widehat{\mathbb{O}}$ efficientively, we hope that it is finite, which is implied by the properties
$\eqref{eqn:i}$, $\eqref{eqn:c}$, $\eqref{eqn:a}$, and $\eqref{eqn:r}$ due to Corollary A.1 \cite{BGSWW}. 
By Proposition \ref{proposition:naturalposet}, $(\widehat{\mathbb{O}},\leq_{\merge})$ is a poset. Hence we have:
\begin{theorem}\label{thm:finiteness}
For each ontology repository $\mathbb{O}$,
if $\widehat{\mathbb{O}}$ satisfies $\eqref{eqn:i}$, $\eqref{eqn:c}$, $\eqref{eqn:a}$, and $\eqref{eqn:r}$,
then $\widehat{\mathbb{O}}$ is a finite poset with the natural order $\leq_{\merge}$ 
given by $\merge$.
\end{theorem}

\begin{remark}
\begin{enumerate}[$1.$]
\item
The study of finiteness conditions of semigroups consists in giving some conditions 
to assure the finiteness of such semigroups.
In this study, one of the conditions which is generally required is that of being {\em finitely generated}.
A semigroup is {\em periodic} if only finitely many elements can be generated from each element $e$, namely,
$\{e, e^2, \cdots, e^n, \cdots\}$ is finite.
{\em Burnside problem for semigroups} asks if a periodic finitely generated semigroup is finite.

This problem has a negative answer in general \cite{lv, mh}.
However, the answer is positive if the property of commutativity or idempotence is required:
If the finitely generated semigroup is commutative, then its finiteness is trivial.
For the finiteness of a finitely generated idempotent semigroup, see \cite{BL}.

\item
Given an ontology merging system $(\mathfrak{O},\sim,\merge)$ and an ontology repository ${\mathbb O}\subseteq \mathfrak{O}$,
$(\mathfrak{O},\merge)$ is a partial groupoid or  semigroup, depending on the associativity of $\merge$.
The finiteness of $\widehat{\mathbb O}$ requires the {\em Burnside problem for partial groupoids or semigroups}.
\end{enumerate}
\end{remark}

\section{Ontology Merging Given by Pushouts}\label{section:pushouts}
Recall that given two sets $X$ and $Y$, the diagram
\begin{equation}\label{diag:setpushout}
\xymatrix{
X\cap Y \ar@{^{(}->}[r]  \ar@{^{(}->}[d] & Y \ar@{^{(}->}[d] \\
X  \ar@{^{(}->}[r] & X\cup Y
}
\end{equation}
with all inclusion functions,
is a pushout in the category ${\bf Set}$ of sets and functions between sets.
If $X\cap Y=\emptyset$, then the pushout square $(\ref{diag:setpushout})$ turns out to be the following pushout:
$$\xymatrix{
\emptyset \ar@{^{(}->}[r]  \ar@{^{(}->}[d] & Y \ar@{^{(}->}[d]\\
X  \ar@{^{(}->}[r] & X\coprod Y
}$$
where $\emptyset$ is the initial object in ${\bf Set}$ and the coproduct $X\coprod Y= X\cup Y$ when $X\cap Y=\emptyset$.

For a bit more general case, let $f:B\rw X$ and $g:B\rw Y$ be two functions.
When $f$ and $g$ are injective, $f(B)$ and $g(B)$ can be viewed as the different labels of $B$ in $X$ and $Y$, respectively.
One can form the disjoint union $X\cupdot Y$ and then identify $f(b)$ with $g(b)$ for $b\in B$ as they are the different labels of $B$ in $X$ and $Y$, respectively. Hence one has the following pushout:
\begin{equation}\label{diag:setpushoutgeneral}
\xymatrix{
B\ar[r]^g \ar[d]_f & Y\ar[d]^{\iota_Y}\\
X\ar[r]^(.4){\iota_X} & X\sqcup_BY\\
}
\end{equation}
where $X\sqcup_BY=X\cupdot Y/\bumpeq$ and $\bumpeq$ is the smallest equivalent relation generated by
$$\{(f(b),g(b))\;|\;b\in B\}.$$

Now let's return to ontology merging operation. Generally, if two ontologies are entirely unrelated, clearly they can be combined by their disjoint union as in $(\ref{diag:setpushout})$.
On the other hand, merging two ontologies which {\em overlap}, up to relabeling,
should lead to a new ontology that identifies the overlapping elements and keeps related elements apart as far as 
it is possible without violating the requirements on the ontology structures \cite{zkeh}. 
This is similar to the case shown in $(\ref{diag:setpushoutgeneral})$ above.
Hence the merge of ontologies $O_1$ and $O_2$ can be obtained by a categorical colimit, pushout, shown in  \cite{ch,hkes,mfh,zkeh}, 
when there exists an ontology $V$-alignment pair between $O_1$ and $O_2$.
Recall the definition of $V$-alignment pairs, in which ontologies $O_1$ and $O_2$ overlap after relabeling $B$ in $O_1$ and $O_2$,
respectively:
\begin{definition}\label{def:valign}
\begin{enumerate}[$1.$]
\item
 An ontology  $V$-{\em alignment pair} $(r_1:B\rw O_1,r_2:B\rw O_2)$ is given by a pair of ontology homomorphisms with a common domain: $r_1:B\rw O_1$ and $r_2:B\rw O_2$, displayed as
$$\xymatrix{
O_1 && O_2\\
& B \ar[ul]^{r_1} \ar[ur]_{r_2}
}$$
\item\label{item:valigncop}
Let the merging system $(\mathfrak{O},\approx, \sqcup)$ be given by
\begin{itemize}
\item
$O_1\approx_B O_2$ if and only if there is an ontology $V$-alignment pair $(r_1:B\rw O_1,r_2:B\rw O_2)$;
\item
If $O_1\approx_B O_2$ then $O_1\sqcup_B O_2$ is given by the pushout:
$$\xymatrix{
B \ar[r]^{r_2}\ar[d]_{r_1} & O_2\ar[d]^{i_{r_1}}\\
O_1 \ar[r]^(0.38){i_{r_2}} & O_1\sqcup_B O_2
}$$
\end{itemize}
\end{enumerate}
\end{definition}

Given two ontologies $O_1$ and $O_2$, notice that both definitions of $\approx_B$ and $\sqcup_B$ in Definition $\ref{def:valign}$ above depend on a base ontology $B$. 
The merge of ontologies $O_1$ and $O_2$ given by pushouts depends on the base ontology $B$. 
We assume that the overlaps of $O_1$ and $O_2$ are identified {\em as many as possible} 
when merging $O_1$ and $O_2$ using pushouts.

If there exists an initial object $I$ in the category of the ontologies concerned,
then there is an ontology $V$-alignment pair $(!_{O_1}:I\rw O_1, !_{O_2}:I\rw O_2)$ and $O_1\sqcup_IO_2=O_1\coprod O_2$.

If there is an ontology homomorphism $h:O_1\rw O_2$, then there is an ontology $V$-alignment pair
$(1_{O_1}:O_1\rw O_1,h:O_1\rw O_2)$ and $O_1\sqcup_{O_1}O_2=O_2$ since
$$\xymatrix{
O_1 \ar@{=}[d] \ar[r]^h & O_2 \ar@{=}[d]\\
O_1 \ar[r]^h & O_2\\
}$$
is a pushout. Clearly, there is a unique ontology homomorphism $[h,1]:O_1\sqcup_IO_2\rw O_1\sqcup_{O_1}O_2=O_2$
as 
$$O_1\sqcup_IO_2=O_1\coprod O_2\rw O_2=O_1\sqcup_{O_1}O_2:$$
$$\xymatrix{
& O_2\\
O_1\ar[ur]^h\ar[r]_(.38){\iota_1} & O_1\coprod O_2 \ar@{..>}[u]|(0.4){\exists ![h,1]} & O_2\ar[l]^(0.38){\iota_2}\ar@{=}[ul]\\
}$$

To understand how an ontology
homomorphism $O_1\sqcup_BO_2\rw O_1\sqcup_{B'}O_2$ is induced after the base being changed from $B$ to $B'$, 
let's define a {\em homomorphism} between ontology $V$-alignment pairs.

\begin{definition}\label{defn:valignmenthom}
A {\em homomorphism} from an ontology $V$-alignment pair $(r_1:B\rw O_1, r_2:B\rw O_2)$ to 
an ontology $V$-alignment pair 
$(r_1':B'\rw O_1', r_2':B'\rw O_2')$
is given by a triple of ontology homomorphisms $(f:B\rw B',f_1:O_1\rw O_1', f_2:O_2\rw O_2')$ such that
$$\xymatrix@=1.4em{
&& O_2 \ar[rrrr]^{f_2}&&&& O_2'\\
O_1 \ar|(.4)\hole[rrrr]_(.6){f_1}&&&& O_1'\\
& B\ar[ul]^{r_1}\ar[uur]_(.7){r_2} \ar[rrrr]^{f} &&&& B'\ar[ul]_{r_1'}\ar[uur]_(.7){r_2'}\\
}$$
commutes.
\end{definition}

By the universal property of a pushout, a homomorphism 
$(f:B\rw B',f_1:O_1\rw O_1', f_2:O_2\rw O_2')$ between ontology $V$-alignment pairs gives rise to a unique ontology homomorphism 
$f_*:O_1\sqcup_BO_2\rw O_1'\sqcup_{B'}O_2'$ making
\begin{equation}\label{diag:mergemap}
\xymatrix@=1em{
& O_1\sqcup_BO_2 \ar@{.>}[rrrr]^{f_*} &&&& O_1'\sqcup_{B'}O_2'\\
&& O_2 \ar[rrrr]^(.4){f_2}\ar[ul]_(0.4){\iota_2} &&&& O_2'\ar[ul]_(0.4){\iota_2'} \\
O_1 \ar[uur]^{\iota_1}\ar|(.486)\hole[rrrr]_(.8){f_1}&&&& O_1'\ar|(.5)\hole[uur]^(.7){\iota_1'}\\
& B\ar[ul]^{r_1}\ar[uur]_(.7){r_2} \ar[rrrr]_{f} &&&& B'\ar[ul]_{r_1'}\ar[uur]_(.6){r_2'}\\
}
\end{equation}
commute.  
Hence an ontology homomorphism $f_*:O_1\sqcup _BO_2\rw  O_1'\sqcup_{B'}O_2'$ is induced
when there is  a homomorphism from $(r_1:B\rw O_1, r_2:B\rw O_2)$ to
 $(r_1':B'\rw O_1', r_2':B'\rw O_2')$.
Therefore, we have:

\begin{proposition}
Each homomorphism $(f:B\rw B',f_1:O_1\rw O_1', f_2:O_2\rw O_2')$ between ontology $V$-alignment pairs
gives rise to a unique ontology homomorphism 
$f_*:O_1\sqcup_BO_2\rw O_1'\sqcup_{B'}O_2'$ making the diagram $(\ref{diag:mergemap})$ commute.
\end{proposition}

If both $f_1$ and $f_2$ are identity homomorphisms in the commutative diagram $(\ref{diag:mergemap})$,  
then we have the following commutative diagram:
\begin{equation}\label{diag:diag:mergemap11}
\xymatrix{
&& O_1\ar[dr]^{\iota_1}\ar@/^1pc/[drr]^{\iota_1'}\\
B\ar[r]|f \ar@/^1pc/[urr]^{r_1}\ar@/_1pc/[drr]_{r_2} & B' \ar[ur]^{r_1'}\ar[dr]_{r_2'}
        && \hspace{-8mm}O_1\sqcup_B O_2 \ar@{.>}[r]|(.36){f_*}
                    & O_1\sqcup_{B'}O_2\\
&& O_2\ar[ur]_{\iota_2}\ar@/_1pc/[urr]_{\iota_2'}\\
}\end{equation}

If $f$ is epic, namely, $xf=yf$ implies $x=y$, then the induced homomorphism $f_*$ is an isomorphism:

\begin{proposition}\label{prop:epiciso}
If $f$ is an epic ontology homomorphism such that
$$\xymatrix{
&& O_1\\
B\ar[urr]^{r_1}\ar[drr]_{r_2}\ar[r]|f & B'\ar[dr]^{r_2'}\ar[ur]_{r_1'}\\
&& O_2\\
}$$
commutes in the category of the ontologies concerned, then $O_1\sqcup_BO_2\cong O_1\sqcup_{B'}O_2$.
\end{proposition}

The following proposition shows that $\leq_{\sqcup}$ can be characterized by the existence of an ontology homomorphism.
\begin{proposition}\label{prop:<=hom}
In the ontology merging system $(\mathfrak{O}, \approx, \sqcup)$,
there is an ontology $B$ such that $O_1\leq_{\sqcup_B} O_2$ if and only if
there is an ontology homomorphism $h:O_1\rw O_2$.
\end{proposition}

By Lemmas \ref{lemma:i}, \ref{lemma:c}, \ref{lemma:ass}, \ref{lemma:r}, and \ref{prop:parsemigroup}, and Theorem \ref{thm:partialordersequal},
clearly
$(\mathfrak{O},\approx, \sqcup)$ is actually an ontology merging system satisfying all properties introduced 
Sections \ref{section:properties} and \ref{section:naturalorder} 
and $(\mathfrak{O},\sqcup)$ is a partial idempotent commutative semigroup. 
\begin{theorem}\label{theorem:pushoutsall}
The merging system $(\mathfrak{O},\approx, \sqcup)$ satisfies $\eqref{eqn:lub}$, $\eqref{eqn:comp}$, $\eqref{eqn:i}$, $\eqref{eqn:c}$,
$\eqref{eqn:a}$, $\eqref{eqn:cass}$, $\eqref{eqn:sass}$, and $\eqref{eqn:r}$
and $(\mathfrak{O},\sqcup)$ is a partial idempotent commutative semigroup.
\end{theorem}

\begin{remark}
Given an ontology merging system $(\mathfrak{O}, \approx, \sqcup)$ and an ontology repository
$\mathbb{O} \subseteq \mathfrak{O}$, the finiteness of $\widehat{\mathbb O}$ can also be obtained by $\eqref{eqn:sass}$.

By Lemma \ref{lemma:ass}, $(\mathfrak{O}, \approx, \sqcup)$ satisfies $\eqref{eqn:sass}$
and so,  by Proposition $\ref{proposition:assrel}(\ref{remark: assextension})$,
the null extension $(\mathfrak{O}^{\uparrow}, \sqcup^{\uparrow})$ is a total semigroup and 
$$\widehat{\mathbb{O}}\cup \{\uparrow\}=\bigcup_{i=1}^{+\infty}\big(\mathbb{O}\cup \{\uparrow\}\big)^i,$$
that is, $\widehat{\mathbb{O}}\cup \{\uparrow\}$ is the idempotent finitely generated semigroup 
given by $\mathbb{O}\cup\{\uparrow\}$.
Hence the finiteness of $\widehat{\mathbb{O}}$ in $(\mathfrak{O}, \approx, \sqcup)$ is replied by the finitess of an idempotent 
finitely generated semigroup $\langle \mathbb{O}\cup \{\uparrow\}\rangle$, which is finite as shown in \cite{BL}.
\end{remark}

\section{Conclusions}\label{section:conclusions}
We introduced ontology merging systems and studied the ontology aligning and merging operations by the properties the operations shared in an ontology merging system, such as,
idempotence $\eqref{eqn:i}$, commutativity $\eqref{eqn:c}$, associativity $\eqref{eqn:a}$, and representativity $\eqref{eqn:r}$
without any specific setting and internal ontology details being considered.

Following the natural partial order approaches on semigroups,
we introduced the natural partial order given by the merging operation to the ontologies concerned
so that the ontologies, in the ontology merging system satisfying $\eqref{eqn:i}$ and $\eqref{eqn:cass}$,
form a poset. Hence the ontologies can be compared, sorted, and selected by sorting and selection algorithms 
in posets and for the transitive relations.

An ontology repository $\mathbb{O}$ is a finite set of the ontologies concerned.
We viewed the repository $\mathbb{O}$ as generators in a partial algebraic environment and 
treated the merging closure of $\mathbb{O}$, the smallest set of ontologies, 
which is closed with respect to merging, as the subalgebraic structure
generated by $\mathbb{O}$.
We proved that the merging closure of a given repository is partially ordered 
and finite under some reasonable conditions: $\eqref{eqn:i}$, $\eqref{eqn:c}$, $\eqref{eqn:a}$, and $\eqref{eqn:r}$, so that we can compute and compare these ontologies generated from the repository by merging.

We also showed that the ontology merging system, given by ontology $V$-alignment pairs and pushouts,
satisfies all properties introduced in Sections \ref{section:properties} and \ref{section:naturalorder} 
and so there are efficient algorithms to 
compute merging closures and query the specific ontologies generated by pushouts.


\begin{appendices}
The minimum requirements of the categorical notions and ontology concepts for the paper include: 
category, homomorphism, isomorphism, coproduct, pullback, pushout, monic, epic, injection,  initial object,
ontology, ontology homomorphism.
For the notions and a systematic introduction to category theory, the reader may consult, for instance, 
\cite{ahs, ml, hkes, zkeh}.

\section{Lemmas}\label{append:lemmas}
Lemmas \ref{lemma:pushoutbyepic} and \ref{pushout:ont} will be shown in a general category.
\begin{lemma}\label{lemma:pushoutbyepic}
Assume that the diagram
$$
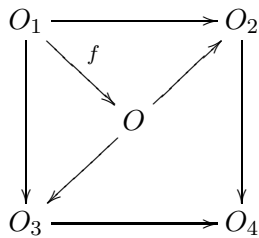
\xymatrix{
O_1 \ar[rr] \ar[dd] \ar[dr]^f && O_2 \ar[dd]\\
& O \ar[ur] \ar[dl]&\\
O_3 \ar[rr]&& O_4\\
}$$
commutes.
\begin{enumerate}[$1.$]
\item
If outer square $O_1,O_2,O_3,O_4$ is a pushout, then so is $O,O_2,O_3,O_4$.
\item
If $O,O_2,O_3,O_4$ is a pushout and $f$ is epic, then the outer square $O_1,O_2,O_3,O_4$ is a pushout.
\end{enumerate}
\end{lemma}
\begin{proof}
Both are clear.
\end{proof}
\begin{lemma}\label{pushout:ont}
$e:B\rw O$ is an epic if and only if
$$\xymatrix{
B\ar[r]^e\ar[d]_e & O\ar[d]^{1_O}\\
O\ar[r]^{1_O} & O
}$$
is a pushout.
\end{lemma}
\begin{proof}
It is routine to verify both directions.
\end{proof}

Lemmas \ref{lemma:i}, \ref{lemma:c}, and \ref{lemma:ass} below are proved on the ontology merging system
$(\mathfrak{O}, \approx, \sqcup)$.
\begin{lemma}\label{lemma:i}
$(\mathfrak{O}, \approx, \sqcup)$ satisfies $\eqref{eqn:i}$.
\end{lemma}
\begin{proof}
Since $1_O:O\rw O$ is epic,
by Lemma \ref{pushout:ont}
$$\xymatrix{
O\ar[r]^{1_O}\ar[d]_{1_O} & O\ar[d]^{1_O}\\
O\ar[r]^{1_O} & O
}$$
is a pushout.
Hence
$O\sqcup_OO=O$ and therefore $\eqref{eqn:i}$ holds. 
\end{proof}

\begin{lemma}\label{lemma:c}
$(\mathfrak{O}, \approx, \sqcup)$ satisfies \eqref{eqn:c}.
\end{lemma}
\begin{proof}
Given an ontology $V$-alignment pair:
$$\xymatrix{
O_1 && O_2\\
& B \ar[ul]^{r_1} \ar[ur]_{r_2}
}$$
we have the following ontology $V$-alignment pair:
$$\xymatrix{
O_2 && O_1\\
& B \ar[ul]^{r_2} \ar[ur]_{r_1}
}$$
and the merges $O_1\sqcup_BO_2$ and $O_2\sqcup_BO_1$ are given by
the following pushouts:
$$\xymatrix{
B \ar[r]^{r_2}\ar[d]_{r_1} & O_2\ar[d]^{i_{r_1}}\\
O_1 \ar[r]^(0.36){i_{r_2}} & O_1\sqcup_B O_2
}$$
and
$$\xymatrix{
B \ar[r]^{r_1}\ar[d]_{r_2} & O_1\ar[d]^{j_{r_2}}\\
O_2\ar[r]^(0.36){j_{r_1}} & O_2\sqcup_B O_1
}$$
respectively.
Hence there exist unique morphisms $j_1:O_1\sqcup_B O_2\rw O_2\sqcup_B O_1$
and $j_2:O_2\sqcup_BO_1\rw O_1\sqcup_BO_2$ making the following
$$\xymatrix{
B \ar[r]^{r_2}\ar[d]_{r_1} & O_2\ar[d]^{i_{r_1}} \ar@/^1pc/[ddr]^{j_{r_1}}\\
O_1\ar@/_1pc/[drr]_{j_{r_2}} \ar[r]^(0.36){i_{r_2}} & O_1\sqcup_B O_2 \ar@{.>}@<0.4ex>[dr]^{j_1} \\
&& O_2\sqcup_B O_1\ar@{.>}@<0.4ex>[ul]^{j_2}
}$$
commutes serially.
Hence
$$(j_2j_1)i_{r_1}=j_2j_{r_1}=i_{r_1}$$
and
$$(j_2j_1)i_{r_2}=j_2j_{r_2}=i_{r_2}$$
and so $j_2j_1=1_{O_1\sqcup_BO_2}$.

Similarly, we have $j_1j_2=1_{O_2\sqcup_B O_1}$.
Thus, $O_1\sqcup_B O_2\cong O_2\sqcup_B O_1$, as desired
\end{proof}

\begin{lemma}\label{lemma:ass}
$(\mathfrak{O}, \approx, \sqcup)$ satisfies $\eqref{eqn:sass}$.
\end{lemma}
\begin{proof}
For all ontologies $O_1,O_2,O_3\in \mathfrak{O}$,
clearly either of $O_1\sqcup (O_2\sqcup O_3)$ and $(O_1\sqcup O_2)\sqcup O_3$ is defined so is the other.
Given ontology $V$-alignment pairs $(r_1:B_1\rw O_1,r_2:B_1\rw O_2)$ and $(r_3:B_2\rw O_2,r_4:B_2\rw O_3)$ 
that form a $W$-alignment diagram, it suffices to show that 
$O_1\sqcup_{B_1} (O_2\sqcup_{B_2} O_3)=(O_1\sqcup_{B_1} O_2)\sqcup_{B_2} O_3$,
up to isomorphism.

Form pushouts $(s_1)$, $(s_2)$, and $(s_3)$ and then the pushout $B_2, O_3, O_1\sqcup_{B_1}O_2, X$:
$$\xymatrix{
&&&& X \ar@{..>}@<0.4ex>[dll]^{\beta}\\
&& O_{123} \ar@{..>}@<0.4ex>[urr]^{\alpha} \\
& O_1\sqcup_{B_1}O_2 \ar[ur]_{j_1}  \ar@/^3pc/[uurrr]^{\iota_1}&& O_2\sqcup_{B_2}O_3\ar[ul]^{j_2}\\
O_1 \ar[ur]^{i_1} && O_2\ar[ul]_{i_2} \ar[ur]_{i_3} \ar@{}[uu]|{(s_3)} && O_3\ar[ul]_{i_4}\ar[uuu]_{\iota_2}\\
& B_1 \ar[ul]^{r_1}\ar[ur]_{r_2} \ar@{}[uu]|{(s_1)} && B_2\ar[ul]^{r_3}\ar[ur]_{r_4} \ar@{}[uu]|{(s_2)}\\
}$$
where $O_{123}=(O_1\sqcup_{B_1}O_2)\sqcup_{O_2}(O_2\sqcup_{B_2}O_3)$
and $X=(O_1\sqcup_{B_1}O_2)\sqcup_{B_2}O_3$.

Since $(s_2)+(s_3)$ is a pushout, there is a unique ontology homomorphism $\alpha: O_{123}\rw X$ such that
$$\iota_1=\alpha j_1\mbox{ and }\iota_2=\alpha j_2i_4.$$
Since $B_2, O_3, O_1\sqcup_{B_1}O_2, X$ is a pushout, there is a unique ontology homomorphism $\beta:X\rw O_{123}$
such that
$$j_1=\beta \iota_1\mbox{ and }j_2i_4=\beta \iota_2.$$
Hence
$$\alpha\beta\iota_1=\alpha j_1=\iota_1\mbox{ and }\alpha\beta\iota_2=\alpha j_2i_4=\iota_2$$
and
$$\beta\alpha j_1=\beta\iota_1=j_1\mbox{ and }\beta\alpha j_2i_4=\beta \iota_2=j_2i_4$$
and therefore
$$\alpha\beta=1_{X}\mbox{ and }\beta\alpha=1_{O_{123}}.$$
So $(O_1\sqcup_{B_1}O_2)\sqcup_{B_2}=X\cong O_{123}.$

Similarly, we have $O_1\sqcup_{B_1}(O_2\sqcup_{B_2}O_3)\cong O_{123}$.
Thus, $O_1\sqcup_{B_1}(O_2\sqcup_{B_2}O_3)=(O_1\sqcup_{B_1}O_2)\sqcup_{B_2}O_3$ up to ontology isomorphism, as desired.
\end{proof}

\begin{lemma}\label{lemma:r}
$(\mathfrak{O}, \approx, \sqcup)$ satisfies \eqref{eqn:r}:
Given ontologies $O_1$, $O_2$, and $O$ such that $O_1\sqcup_{B} O_2$ exists, $O\approx O_1$ implies 
$O\approx O_1\sqcup_B O_2$ and $O_1\approx O$ implies 
$O_1\sqcup_B O_2\approx O$.
\end{lemma}
\begin{proof}
Since $O_1\sqcup_{B} O_2$ exists and $O\approx O_1$, $O_1\sqcup _BO_2=X$ is given by
the pushout $(*)$ and there is an ontology $V$-alignment pair $(s_1:B'\rw O, s_2:B'\rw O_1)$:
$$
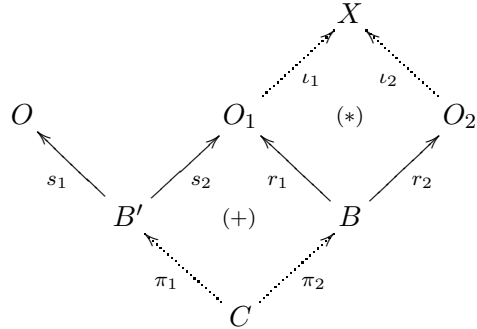
\xymatrix{
&&& X\\
O && O_1 \ar@{..>}[ur]_{\iota_1}  \ar@{}[rr]|{(*)} && O_2\ar@{..>}[ul]^{\iota_2}\\
& B' \ar[ul]^{s_1}\ar[ur]_{s_2} \ar@{}[rr]|{(+)}&& B\ar[ul]^{r_1}\ar[ur]_{r_2}\\
&& C\ar@{..>}[ul]^{\pi_1}\ar@{..>}[ur]_{\pi_2}\\
}$$
here $(+)$ is a pullback.
Hence there is an ontology $V$-alignment pair:
$$
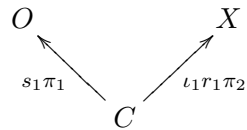
\xymatrix{
O && X \\
& C \ar[ul]^{s_1\pi_1} \ar[ur]_{\iota_1r_1\pi_2}
}$$
and therefore
$O\approx O_1\sqcup_B O_2$.
\end{proof}

\begin{lemma}\label{prop:parsemigroup}
$(\mathfrak{O},\sqcup)$ is a partial idempotent commutative semigroup.
\end{lemma}
\begin{proof}
By Lemmas \ref{lemma:i}, \ref{lemma:c}, and \ref{lemma:ass}.
\end{proof}

\section{Proofs}\label{append:theorems}
\noindent {\bf Proposition \ref{proposition:assrel}.}
{\em
Let $(\mathfrak{O}, \sim,\merge)$ be an ontology merging system.
\begin{enumerate}[$(i)$]
\item\label{remark: assimplya}
If $(\mathfrak{O},\sim,\merge)$ satisfies $\eqref{eqn:sass}$, then $(\mathfrak{O},\sim,\merge)$ satisfies $\eqref{eqn:a}$;
\item\label{remark: cassimplya}
If $(\mathfrak{O},\sim,\merge)$ satisfies $\eqref{eqn:cass}$, then $(\mathfrak{O},\sim,\merge)$ satisfies $\eqref{eqn:a}$;
\item\label{remark: assequiv}
If $(\mathfrak{O},\sim,\merge)$ satisfies $\eqref{eqn:c}$, then
$\eqref{eqn:cass}$ is equivalent to $\eqref{eqn:a}$ and $\eqref{eqn:r}$;
\item\label{remark: assextension}
$(\mathfrak{O},\sim,\merge)$ satisfies $\eqref{eqn:sass}$ if and only if the null extension $(\mathfrak{O}^{\uparrow}, \sim^{\uparrow}, \merge^{\uparrow})$ is a total ontology merging system such that
$(O_1\merge^{\uparrow}O_2) \merge^{\uparrow}O_3=O_1 \merge^{\uparrow}(O_2 \merge^{\uparrow} O_3)$ for all $O_1,O_2,O_3\in \mathfrak{O}^{\uparrow}$, namely, $(\mathfrak{O}^{\uparrow},\merge^{\uparrow})$ is a semigroup.
\end{enumerate}
}
\begin{proof}
$(\ref{remark: assimplya})$ and $(\ref{remark: cassimplya})$ are obvious.

$(\ref{remark: assequiv})$: 
Suppose that  $\eqref{eqn:cass}$ holds true. By $(\ref{remark: cassimplya})$, $\eqref{eqn:a}$ follows.
If $O_3=O_1\merge O_2$ and $O\sim O_1$, then, by $\eqref{eqn:cass}$,
$$O\merge (O_1\merge O_2)=(O\merge O_1)\merge O_2\neq \;\uparrow.$$
Hence $O\sim O_1\merge O_2$ and therefore $\eqref{eqn:rr}$ is satisfied.
Similarly, $\eqref{eqn:rl}$ is also satisfied. Thus, $\eqref{eqn:r}$ holds true.

Conversely, suppose that both $O_1\merge O_2$ and $O_2\merge O_3$ exist. By $\eqref{eqn:r}$ and $\eqref{eqn:c}$,
$O_1\merge O_2\sim O_3$ and $O_1\sim O_2\merge O_3$. Hence
$(O_1\merge O_2)\merge O_3$ and $O_1\merge (O_2\merge O_3)$ exist and therefore, by $\eqref{eqn:a}$,
$$(O_1\merge O_2)\merge O_3=O_1\merge (O_2\merge O_3)\neq\;\uparrow.$$
So $\eqref{eqn:cass}$ holds true.

$(\ref{remark: assextension})$:
If $(\mathfrak{O}^{\uparrow}, \merge^{\uparrow})$ is a semigroup, namely, $\merge^{\uparrow}$ is an associative operation
on $\mathfrak{O}^{\uparrow}$, then, for all $O_1,O_2,O_3\in 
\mathfrak{O}\subseteq \mathfrak{O}^{\uparrow}$,
$$(O_1\merge^{\uparrow}O_2)\merge^{\uparrow}O_3=O_1\merge^{\uparrow}(O_2\merge^{\uparrow}O_3)
=O\in \mathfrak{O}^{\uparrow}.$$
If $O=\;\uparrow$, then 
$$(O_1\merge O_2)\merge O_3=O_1\merge(O_2\merge O_3)
=\;\uparrow.$$
If $O\neq\;\uparrow$, then 
$$(O_1\merge O_2)\merge O_3=O_1\merge(O_2\merge O_3)
\neq\;\uparrow.$$
Hence 
$(\mathfrak{O},\sim,\merge)$ satisfies $\eqref{eqn:sass}$.

Conversely, suppose now that $(\mathfrak{O},\sim,\merge)$ satisfies $\eqref{eqn:sass}$.
For all $O_1,O_2,O_3\in\mathfrak{O}^{\uparrow}$,
if $\uparrow\in\{O_1,O_2,O_3\}$, then
$$(O_1\merge^{\uparrow}O_2)\merge^{\uparrow}O_3=O_1\merge^{\uparrow}(O_2\merge^{\uparrow}O_3)
=\;\uparrow\;\in \mathfrak{O}^{\uparrow}.$$

If $\uparrow\;\notin\{O_1,O_2,O_3\}$, then, by $\eqref{eqn:sass}$,
$$(O_1\merge O_2)\merge O_3=O_1\merge(O_2\merge O_3)
=\;\uparrow\;\mbox{ or } \neq \;\uparrow.$$
Hence
$$(O_1\merge^{\uparrow}O_2)\merge^{\uparrow}O_3=O_1\merge^{\uparrow}(O_2\merge^{\uparrow}O_3)
\in \mathfrak{O}^{\uparrow}$$
and therefore $\merge^{\uparrow}$ is associative. Thus,
$(\mathfrak{O}^{\uparrow}, \merge^{\uparrow})$ is a semigroup, as desired.
\end{proof}

\noindent {\bf Proposition \ref{proposition:naturalposet}.}
{\em If $(\mathfrak{O}, \sim, \merge)$ satisfies $\eqref{eqn:i}$ and $\eqref{eqn:cass}$,
then $\leq_{\merge}$ is a partial order on $\mathfrak{O}$ and so $(\mathfrak{O},\leq_{\merge})$ is a poset.
}
\begin{proof}
It suffices to prove that $\leq_{\merge}$ is reflexive, transitive, and antisymmetric.

By $\eqref{eqn:i}$, $O\merge O = O$ and so $O\leq_{\merge} O$ for all $O\in \mathfrak{O}$.
Hence $\leq_{\merge}$ is reflexive.

Assume that $O_1\leq_{\merge} O_2$ and $O_2\leq_{\merge} O_3$. 
Then $O_1\sim O_2$, $O_2\sim O_1$, $O_2\sim O_3$, $O_3\sim O_2$, 
$O_1\merge O_2 =O_2\merge O_1= O_2$, and $O_2\merge O_3 = O_3\merge O_2 = O_3$. By $\eqref{eqn:cass}$,
$$O_1\merge O_3=O_1\merge (O_2\merge O_3)=(O_1\merge O_2)\merge O_3=O_2\merge O_3=O_3$$
and
$$O_3\merge O_1= (O_3\merge O_2)\merge O_1=O_3\merge (O_2\merge O_1)=O_3\merge O_2=O_3.$$
Hence $O_1\sim O_3$, $O_3\sim O_1$, and $O_1\leq_{\merge} O_3$ and therefore $\leq_{\merge}$ is transitive.

If $O_1\leq_{\merge} O_2$ and $O_2\leq_{\merge} O_1$, then
$O_1\merge O_2 = O_2\merge O_1=O_2$ and $O_1\merge O_2=O_2\merge O_1 = O_1$ and so
$O_1=O_1\merge O_2 = O_2$. Thus,
$\leq_{\merge}$ is antisymmetric.
\end{proof}

\noindent {\bf Theorem \ref{thm:partialordersequal}.}
{\em Let $\sim$ be reflexive and  commutative and $\preceq$ a partial order on $(\mathfrak{O},\sim,\merge)$. Then  
$(\mathfrak{O},\sim,\merge)$ satisfies $\eqref{eqn:lub}$ and $\eqref{eqn:comp}$ with respect to $\preceq$ if and only if 
$(\mathfrak{O},\sim,\merge)$ satisfies $\eqref{eqn:i}$, $\eqref{eqn:c}$, $\eqref{eqn:a}$, 
and $\eqref{eqn:r}$ and $\preceq\;=\;\leq_{\merge}$.
}
\begin{proof}
``only if": For each ontology $O\in \mathfrak{O}$, since $\sim$ is reflexive, $O\sim O$.
By $\eqref{eqn:lub}$, $O\preceq O\merge O$.
Since $O\preceq O$ and $O\merge O$ is the least upper bound of $O$ and $O$, 
$O\merge O\preceq O$. Since $\preceq$ is a partial order, $O\merge O = O$. Hence $\eqref{eqn:i}$ holds.

If $O_1\sim O_2$, then, clearly, $O_2\sim O_1$ as $\sim$ is commutative. Since 
$O_2\preceq O_1\merge O_2\mbox{ and }O_1\preceq O_1\merge O_2$, by $\eqref{eqn:lub}$ 
we have $O_2\merge O_1\preceq O_1\merge O_2$.
Similarly, $O_1\merge O_2\preceq O_2\merge O_1$. Then $O_1\merge O_2=O_2\merge O_1$ and so
$\eqref{eqn:c}$ is satisfied.

Suppose that $O_1\merge O_2$ exists and $O_1\sim O$. Since $O_1\preceq O_1\merge O_2$, $O_1\merge O_2\sim O$ 
by $\eqref{eqn:rcomp}$.
Hence $\eqref{eqn:rr}$ holds. Similarly, $\eqref{eqn:rl}$ holds true. Thus, $\eqref{eqn:r}$ is satisfied.

Assume that both $(O_1\merge O_2)\merge O_3$ and $O_1\merge (O_2\merge O_3)$ exist. 
Since $O_3\preceq O_2\merge O_3$, by $\eqref{eqn:lcomp}$,
$$(O_1\merge O_2)\merge O_3\preceq (O_1\merge O_2)\merge (O_2\merge O_3).$$
Clearly, $O_2\merge O_3\preceq (O_1\merge O_2)\merge O_3$
and $O_1\merge O_2\preceq (O_1\merge O_2)\merge O_3$. 
So $(O_1\merge O_2)\merge (O_2\merge O_3)\preceq (O_1\merge O_2)\merge O_3$ by $\eqref{eqn:lub}$. 
Hence $(O_1\merge O_2)\merge O_3 = (O_1\merge O_2)\merge (O_2\merge O_3)$.
Symmetrically, $O_1\merge (O_2\merge O_3) = (O_1\merge O_2)\merge (O_2\merge O_3)$. Then 
$O_1\merge (O_2\merge O_3) = (O_1\merge O_2)\merge O_3$ and so $\eqref{eqn:a}$ is satisfied.

Assume now that $O_1\leq_{\merge} O_2$. Then $O_2=O_1\merge O_2=O_2\merge O_1$. Hence
$O_1\preceq O_1\merge O_2= O_2$ by $\eqref{eqn:lub}$. On the other hand, suppose that $O_1\preceq O_2$. Then,
by $\eqref{eqn:lub}$, $O_2\preceq O_1\merge O_2$. Hence $O_2=O_1\merge O_2$ 
as $O_1\merge O_2\preceq O_2\merge O_2 = O_2$ 
by $\eqref{eqn:rcomp}$.
Similarly, $O_2=O_2\merge O_1$. Thus, $O_1\leq_{\merge} O_2$, as desired.
So $\preceq \;=\; \leq_{\merge}$.

``if": Suppose now that $\preceq \;=\; \leq_{\merge}$ and $(\mathfrak{O},\sim,\merge)$ satisfies 
$\eqref{eqn:i}$, $\eqref{eqn:c}$, $\eqref{eqn:a}$, 
and $\eqref{eqn:r}$. 
By Proposition $\ref{proposition:assrel}(\ref{remark: assequiv})$,
$(\mathfrak{O},\sim, \merge)$ satisfies $\eqref{eqn:cass}$
and so $\leq_{\merge}$ is a partial order by Proposition $\ref{proposition:naturalposet}$.

Let $O_1,O_2\in \mathfrak{O}$ be such that $O_1\sim O_2$. 
Then $O_1\merge O_2$ and $O_2\merge O_1$ exist and equal by $\eqref{eqn:c}$. Since
$$O_1\merge (O_1\merge O_2)=(O_1\merge O_1)\merge O_2=O_1\merge O_2$$
and
$$(O_1\merge O_2)\merge O_1=O_1\merge (O_2\merge O_1)=O_1\merge (O_1\merge O_2)=(O_1\merge O_1)\merge O_2=O_1\merge O_2,$$
we have $O_1\leq_{\merge} O_1\merge O_2$. Similarly, $O_2\leq_{\merge} O_1\merge O_2$. 
Hence $O_1\merge O_2$ is an upper bound of $O_1$ and $O_2$.

Assume that $O_1\leq_{\merge} O$ and $O_2\leq_{\merge} O$. 
Then $O_1\merge O=O\merge O_1=O$ and $O_2\merge O=O\merge O_2=O$. Hence
$$O=O_1\merge O=O_1\merge (O_2\merge O)=(O_1\merge O_2)\merge O$$
and 
$$O=O\merge O_1= (O\merge O_2)\merge O_1=O\merge (O_2\merge O_1)= O\merge (O_1\merge O_2)$$
and therefore $O_1\merge O_2\leq_{\merge} O$. Then $O_1\merge O_2$ is the least upper bound of $O_1$ and $O_2$
and so $\eqref{eqn:lub}$ holds.

For all ontologies $O_1$, $O_2$, $O$ such that $O_1\leq_{\merge} O_2$ and $O\sim O_1$, 
obviously $O\sim O_1\merge O_2=O_2$ by $\eqref{eqn:r}$. Hence $\eqref{eqn:rcomp}$ is satisfied.
Similarly, $\eqref{eqn:lcomp}$ is also satisfied. Thus, $\eqref{eqn:comp}$ holds true.
\end{proof}

\noindent {\bf Theorem \ref{thm:mergeclosure}.}
{\em Given a repository $\mathbb{O}\subseteq \mathfrak{O}$. 
the merging closure $\widehat{\mathbb{O}}$ exists and it is unique.
}
\begin{proof}
Let 
$$\mathcal{C}\stackrel{\textrm{\scriptsize def}}{=}
\big\{ \mathbb{P} \;|\;\mathbb{O}\subseteq \mathbb{P}\mbox{ and }\mathbb{P}\mbox{ is closed with respect to merging}\big\}.$$
Clearly, $\mathfrak{O}\in \mathcal{C}\neq \emptyset$ and $\mathcal{C}$ is partially ordered by $\subseteq$.
For each chain $\cdots\subseteq \mathbb{P}_{\alpha}\subseteq \cdots\subseteq \mathbb{P}_{\beta}\cdots$, $\cap\mathbb{P}_{\alpha}\in \mathcal{C}$ and it is a lower bound of the totally ordered subset 
$\{\cdots,\mathbb{P}_{\alpha},\cdots,\mathbb{P}_{\beta},\cdots\}$.
Thus, $\mathcal{C}$ has a minimal element, 
$\cap_{\mathbb{P}\in\mathcal{C}}\mathbb{P}$.

Let $\mathbb{P}_1$ and $\mathbb{P}_2$ be two minimal elements for $(\mathcal{C},\subseteq)$.
Then $\mathbb{P}_1\cap \mathbb{P}_2$ is closed with respect to ontology merging $\merge$ and smaller than
$\mathbb{P}_1$ and $\mathbb{P}_2$ and so
$\mathbb{P}_1=\mathbb{P}_1\cap \mathbb{P}_2=\mathbb{P}_2$.
The uniqueness follows. The unique minimal element of $(\mathcal{C},\subseteq)$ is 
the merging closure $\widehat{\mathbb{O}}$:
$$\widehat{\mathbb{O}}=\bigcap_{\mathbb{P}\in\mathcal{C}}\mathbb{P}=
\bigcap\{ \mathbb{P} \;|\;\mathbb{O}\subseteq \mathbb{P}\mbox{ and }\mathbb{P}\mbox{ is closed with respect to merging}\}\neq \emptyset.$$
\end{proof}

\noindent {\bf Proposition \ref{prop:epiciso}.}
{\em If $f$ is an epic ontology homomorphism such that
\begin{equation}\label{diag:vpairmap}
\xymatrix{
&& O_1\\
B\ar[urr]^{r_1}\ar[drr]_{r_2}\ar[r]|f & B'\ar[dr]^{r_2'}\ar[ur]_{r_1'}\\
&& O_2\\
}
\end{equation}
commutes in the category of the ontologies concerned, then $O_1\sqcup_BO_2\cong O_1\sqcup_{B'}O_2$.
}
\begin{proof}
Since the diagram $(\ref{diag:vpairmap})$ commutes, 
$(1_{O_1},1_{O_2},f):(r_1:B\rw O_1, r_2:B\rw O_2)\rw (r_1':B'\rw O_1, r_2':B'\rw O_2)$ is
a homomorphism between ontology $V$-alignment pairs. Hence there is a unique ontology homomorphism 
$f_*:O_1\sqcup_BO_2\rw O_1\sqcup_{B'}O_2$ such that the diagram 
$(\ref{diag:diag:mergemap11})$ commutes.
If $f$ is epic, then, by Lemma \ref{lemma:pushoutbyepic},
$B', O_1,O_1\sqcup_B O_2,O_2$ is commutative and a pushout and so there is a unique ontology homomorphism 
$g:O_1\sqcup_{B'}O_2\rw O_1\sqcup_BO_2$ such that
$$g\iota_1'=\iota_1\mbox{ and }g\iota_2'=\iota_2.$$
$$
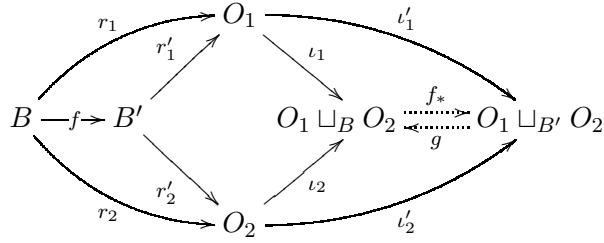
\xymatrix{
&& O_1\ar[dr]^{\iota_1}\ar@/^1pc/[drr]^{\iota_1'}\\
B\ar[r]|f \ar@/^1pc/[urr]^{r_1}\ar@/_1pc/[drr]_{r_2} & B' \ar[ur]^{r_1'}\ar[dr]_{r_2'}
        && \hspace{-8mm}O_1\sqcup_B O_2 \ar@<0.6ex>@{.>}[r]^(.4){f_*} 
                    & O_1\sqcup_{B'}O_2  \ar@<0.6ex>@{.>}[l]^(0.6){g} \\
&& O_2\ar[ur]_{\iota_2}\ar@/_1pc/[urr]_{\iota_2'}\\
}$$
Hence 
$$f_*g\iota_1'=f_*\iota_1=\iota_1'\mbox{ and }f_*g\iota_2'=f_*\iota_2=\iota_2'$$
and therefore $f_*g=1_{O_1\sqcup_{B'}O_2}$.
Similarly, we have $gf_*=1_{O_1\sqcup_BO_2}$ from $f_*\iota_1=\iota_1'\mbox{ and }f_*\iota_2=\iota_2'$.
Then $f_*$ is an ontology isomorphism and so $O_1\sqcup_BO_2\cong O_1\sqcup_{B'}O_2$.
\end{proof}

\noindent {\bf Proposition \ref{prop:<=hom}.}
{\em In the ontology merging system $(\mathfrak{O}, \approx, \sqcup)$,
there is an ontology $B$ such that $O_1\leq_{\sqcup_B} O_2$ if and only if
there is an ontology homomorphism $h:O_1\rw O_2$.
}
\begin{proof}
``{\em if}": Suppose that there exists an ontology homomorphism $h:O_1\rw O_2$. Then
we have a commutative square:
\begin{equation}\label{equ:pushout<=}
\xymatrix{
O_1\ar[r]^h \ar@{=}[d] & O_2\ar@{=}[d]\\
O_1 \ar[r]^h & O_2
}
\end{equation}

For any ontology homomorphisms $x_1:O_1\rw X$ and $x_2:O_2\rw X$ such that $x_1=x_2h$, clearly there is a unique 
ontology homomorphism $x_2:O_2\rw X$ making
$$\xymatrix{
O_1\ar[r]^h \ar@{=}[d] & O_2\ar@{=}[d]\ar@/^/[ddr]^{x_2}\\
O_1 \ar[r]^h \ar@/_/[drr]_{x_1}& O_2\ar@{..>}[dr]|{\exists !x_2}\\
&& X
}$$
commute. Hence $(\ref{equ:pushout<=})$ is a pushout and therefore $O_1\sqcup_{O_1} O_2=O_2\sqcup_{O_1} O_1=O_2$. Thus, there exists an ontology $O_1$ such that $O_1\leq_{\sqcup_{O_1}} O_2$, as desired.

``{\em only if}":
Assume now that $O_1\leq_{\sqcup_B}O_2$.
Then $O_1\approx_B O_2$ and $O_1\sqcup_B O_2=O_2\sqcup_B O_1=O_2$ and so there is a pushout 
in the category of the ontologies concerned:
$$\xymatrix{
B\ar[r]^{r_2} \ar[d]_{r_1} & O_2\ar@{=}[d]\\
O_1 \ar[r]^h & O_2
}$$
Hence there exists an ontology homomorphism $h:O_1\rw O_2$.
\end{proof}

\end{appendices}


\end{document}